\documentclass[runningheads]{llncs}
\usepackage{orcidlink}

\usepackage{graphicx}
\usepackage{color}

\usepackage{amsmath, amssymb, mathtools}
\usepackage{enumitem, cleveref, nameref} %
\usepackage{cancel}
\usepackage{nicefrac}
\usepackage{stmaryrd} %

\usepackage{booktabs}
\usepackage[bottom]{footmisc} %
\usepackage{xargs} %
\usepackage[colorinlistoftodos,prependcaption,textsize=tiny]{todonotes}

\DeclareMathOperator{\E}{\mathbb{E}}

\DeclareMathOperator{\cC}{\mathcal{C}}

\newcommand{\dop}{\mathit{do}}

\newcommand{\PA}{\text{PA}}

\newcommand{\veryshortrightarrow}{\!\!\shortrightarrow\!\!}

\newcommandx{\warning}[2][1=]{\todo[linecolor=red,backgroundcolor=red!25,bordercolor=red,#1]{#2}}
\newcommandx{\mynote}[2][1=]{\todo[linecolor=blue,backgroundcolor=blue!25,bordercolor=blue,#1]{#2}}
\newcommandx{\myquestion}[2][1=]{\todo[linecolor=purple,backgroundcolor=purple!25,bordercolor=purple,#1]{#2}}
\newcommandx{\original}[2][1=]{\todo[linecolor=green,backgroundcolor=green!25,bordercolor=green,#1]{OG!: #2}}

\begin{document}
\title{Causal Entropy and Information Gain for Measuring Causal Control}

\author{Francisco Nunes~Ferreira~Quialheiro~Simoes\orcidlink{0009-0004-9897-6446} \and
Mehdi Dastani\orcidlink{0000-0002-4641-4087} \and
Thijs van Ommen\orcidlink{0000-0001-7593-6258}}

\authorrunning{F. Simoes et al.}
\institute{Department of Information and Computing Sciences, Utrecht University\\
  \email{\{f.simoes,m.m.dastani,t.vanommen\}@uu.nl}
}
\maketitle              %

\begin{abstract}
  Artificial intelligence models and methods commonly lack causal interpretability.
  Despite the advancements in interpretable machine learning (IML) methods, they frequently assign importance to features which lack causal influence on the outcome variable.
  Selecting causally relevant features among those identified as relevant by these methods, or even before model training, would offer a solution.
  Feature selection methods utilizing information theoretical quantities have been successful in identifying statistically relevant features.
  However, the information theoretical quantities they are based on do not incorporate causality, rendering them unsuitable for such scenarios.
  To address this challenge, this article proposes information theoretical quantities that incorporate the causal structure of the system, which can be used to evaluate causal importance of features for some given outcome variable.
  Specifically, we introduce causal versions of entropy and mutual information, termed causal entropy and causal information gain, which are designed to assess how much control a feature provides over the outcome variable.
  These newly defined quantities capture changes in the entropy of a variable resulting from interventions on other variables.
  Fundamental results connecting these quantities to the existence of causal effects are derived.
  The use of causal information gain in feature selection is demonstrated, highlighting its superiority over standard mutual information in revealing which features provide control over a chosen outcome variable.
  Our investigation paves the way for the development of methods with improved interpretability in domains involving causation.
\end{abstract}

\section{Introduction}
\label{sec:introduction}

Causality plays an important role in enhancing not only the prediction power of a model \cite{scholkopf2012causal} but also its interpretability \cite{confalonieri2021historical}.
Causal explanations are more appropriate for human understanding than purely statistical explanations \cite{miller2019explanation}.
Accordingly, comprehending the causal connections between the variables of a system can enhance the interpretability of interpretable machine learning (IML) methods themselves.

Interpretable models such as linear regression or decision trees do not, despite their name, always lend themselves to \emph{causal} interpretations.
To illustrate this point, consider running multilinear regression on the predictors $X_{1}, X_{2}$ and outcome $Y$ within a system whose variables are causally related as depicted in the graph of \Cref{fig:example-scm}.
The regression coefficients $\beta_{1}$ and $\beta_{2}$ of $X_{1}$ and $X_{2}$ might yield large values, which may be (and are often in practice) interpreted as suggesting a causal relationship.
However, a causal interpretation of $\beta_{1}$ would not be appropriate.
Although $X_{1}$ might provide predictive power over $Y$, this does not imply a causal relationship, since this predictive power is due to the confounder $W$.
Consequently, intervening on $X_{1}$ would not impact the outcome $Y$.

In current model-agnostic methods, a causal interpretation is often desirable but rarely possible.
In partial dependence plots (PDPs) \cite{friedman2001pdp}, the partial dependence of a model outcome $\hat{Y}$ on a variable $X_{i}$ coincides with the backdoor criterion formula \cite{pearl2016primer} when the conditioning set encompasses all the other covariates $X_{j\ne i}$ \cite{zhao2019causal}.
Consequently, there is a risk of disregarding statistical dependence or, conversely, finding spurious dependence, by conditioning on causal descendants of $X_{i}$ \cite{zhao2019causal}.
Therefore, PDPs (along with the closely related individual conditional expectation (ICE) lines \cite{goldstein2015peeking}) generally lack a causal interpretation.
Similarly, when utilizing (Local Interpretable Model-Agnostic Explanations) LIME \cite{ribeiro2016lime} to evaluate the importance of a feature for an individual, a causal interpretation cannot be guaranteed.
LIME fits a local model around the point of interest and assesses which features, when perturbed, would cause the point to cross the decision boundary of the model.
However, intervening on a feature in such a way as to cross the model's decision boundary does not guarantee an actual change in the outcome in reality.
This is because the model was trained on observational data, and that feature may merely be correlated with the outcome through a confounding factor, for example, rather than having a causal effect on the outcome.

In both cases just described, it is the presence of confounders, selection bias, or an incorrect direction of causality seemingly implied by the model that can lead to misleading predictions and interpretations.
We need a way to select which features are causally relevant --- \emph{i.e.} give us control over the chosen outcome variable.
Information theoretical quantities such as mutual information are often used to assess the relevance of a feature with respect to a given outcome variable \cite{vergara2014review,beraha2019feature,zhou2022feature}, but this relevance is still purely statistical.
This is a common issue when using standard information theoretical quantities in situations that require consideration of the underlying causal relationships.
A version of mutual information which takes into account the causal structure of the system would solve this problem.
This is what we set out to develop in this work.

In our research, we extend traditional conditional entropy and mutual information to the realm of \emph{interventions}, as opposed to simple conditioning.
This extension drew inspiration from the conceptual and philosophical work presented in\footnotemark \cite{griffiths2015specificity}.
\footnotetext{The reader is referred to \Cref{sec:related-work} for a detailed discussion about this.}
We dub these constructs ``causal entropy'' and ``causal information gain''.
They are designed to capture changes in the entropy of a given variable in response to manipulations affecting other variables.
We derive fundamental results connecting these quantities to the presence of causal effect.
We end by illustrating the use of causal information gain in selecting a variable which allows us to control an outcome variable, and contrast it with standard mutual information.

The novelty of our work consists of providing rigorous definitions for causal entropy and causal information gain, as well as deriving some of their key properties for the first time.
These contributions set the foundations for the development of methods which correctly identify features which provide causal control over an outcome variable.

This paper is organized as follows. In \Cref{sec:formal-setting}, we introduce the definitions of quantities from the fields of causal inference and information theory that will be used throughout the rest of the paper.
\Cref{sec:example} includes a simple example of a structural causal model where standard entropy and mutual information are inadequate for obtaining the desired causal insights.
In \Cref{sec:causal-entr}, we define causal entropy and explore its relation to total effect.
\Cref{sec:causal-inf-gain} discusses the definition of causal information gain and investigates its connection with causal effect.
Furthermore, it revisits the example from \Cref{sec:example}, showing that causal entropy and causal information gain allow us to arrive at the correct conclusions about causal control.
In \Cref{sec:related-work}, we compare the definitions and results presented in this paper with those of previous work.
Finally, in \Cref{sec:disc-concl}, we discuss the obtained results and propose future research directions.

\section{Formal Setting}
\label{sec:formal-setting}
In this section we present the definitions from causal inference and information theory which are necessary for the rest of this paper.
All random variables are henceforth assumed to be discrete and have finite range.

\subsection{Structural Causal Models}
\label{sec:struct-caus-models}

One can model the causal structure of a system by means of a ``structural causal model'', which can be seen as a Bayesian network \cite{koller2009probabilistic} whose graph $G$ has a causal interpretation and each conditional probability distribution (CPD) $P(X_{i} \mid \PA_{X_{i}})$ of the Bayesian network stems from a deterministic function $f_{X_{i}}$ (called ``structural assignment'') of the parents of $X_{i}$.
In this context, it is common to separate the parent-less random variables (which are called ``exogenous'' or ``noise'' variables) from the rest (called ``endogenous'' variables).
Only the endogenous variables are represented in the structural causal model graph.
As is commonly done \cite{peters2017elements}, we assume that the noise variables are jointly independent and that exactly one noise variable $N_{X_{i}}$ appears as an argument in the structural assignment $f_{X_{i}}$ of $X_{i}$.
In full rigor\footnotemark \cite{peters2017elements}:

\footnotetext{We slightly rephrase the definition provided in \cite{peters2017elements} to enhance its clarity. \label{fn:def}}

\begin{definition}[Structural Causal Model]
  \label{def:scm}
  Let $X$ be a random variable with range $R_{X}$ and $\mathbf{W}$ a random vector with range $R_{\mathbf{W}}$.
  A \emph{structural assignment for $X$ from $\mathbf{W}$} is a function $f_{X}\colon R_{\mathbf{W}} \to R_{X}$.
  A \emph{structural causal model (SCM)} $\mathcal{C} = (\mathbf{X}, \mathbf{N}, S, p_{\mathbf{N}})$ consists of:
  \begin{enumerate}
    \item A random vector $\mathbf{X} = (X_{1}, \ldots, X_{n})$ whose variables we call \emph{endogenous}.
    \item A random vector $\mathbf{N} = (N_{X_{1}}, \ldots, N_{X_{n}})$ whose variables we call \emph{exogenous} or \emph{noise}.
    \item A set $S$ of $n$ structural assignments $f_{X_{i}}$ for $X_{i}$ from ($\PA_{X_{i}}, N_{X_{i}}$), where $\PA_{X_{i}} \subseteq \mathbf{X}$ are called \emph{parents} of $X_{i}$.
      The \emph{causal graph} $G^{\mathcal{C}}\vcentcolon=(\mathbf{X}, E)$ of $\mathcal{C}$ has as its edge set $E = \{(P, X_{i}) : X_{i} \in \mathbf{X},\  P\in \PA_{X_{i}}\}$.
      The $\PA_{X_{i}}$ must be such that the $G^{\mathcal{C}}$ is a directed acyclic graph (DAG).
    \item A jointly independent probability distribution $p_{\mathbf{N}}$ over the noise variables. We call it simply the \emph{noise distribution}.
  \end{enumerate}
\end{definition}

  We denote by $\cC(\mathbf{X})$ the set of SCMs with vector of endogenous variables $\mathbf{X}$.
  Furthermore, we write $X \vcentcolon= f_{X}(X, N_{X})$ to mean that $f_{X}(X, N_{X})$ is a structural assignment for $X$.

  Notice that for a given SCM the noise variables have a known distribution $p_{\mathbf{N}}$ and the endogenous variables can be written as functions of the noise variables.
  Therefore the distributions of the endogenous variables are themselves determined if one fixes the SCM.
  This brings us to the notion of the entailed distribution\footref{fn:def} \cite{peters2017elements}:

\begin{definition}[Entailed distribution]
  Let $\mathcal{C} = (\mathbf{X}, \mathbf{N}, S, p_{\mathbf{N}})$ be an SCM. Its \emph{entailed distribution} $p^{\mathcal{C}}_{\mathbf{X}}$  is the unique joint distribution over $\mathbf{X}$ such that $\forall X_{i} \in \mathbf{X},\ X_{i} = f_{X_{i}}(\PA_{X_{i}}, N_{X_{i}})$.
  It is often simply denoted by $p^{\cC}$.
  Let $\mathbf{x}_{-i}\vcentcolon= (x_{1}, \ldots, x_{i-1}, x_{i+1}, \ldots, x_{n})$.
  For a given $X_{i} \in \mathbf{X}$, the marginalized distribution $p^{\cC}_{X_{i}}$ given by $p^{\cC}_{X_{i}}(x_{i}) = \sum_{\mathbf{x}_{-i}} p^{\cC}_{\mathbf{X}}(\mathbf{x})$ is also referred to as \emph{entailed distribution (of $X_{i}$)}.
\end{definition}

An SCM allows us to model interventions on the system.
The idea is that an SCM represents how the values of the random variables are generated, and by intervening on a variable we are effectively changing its generating process.
Thus intervening on a variable can be modeled by modifying the structural assignment of said variable, resulting in a new SCM differing from the original only in the structural assignment of the intervened variable, and possibly introducing a new noise variable for it, in place of the old one.
Naturally, the new SCM will have an entailed distribution which is in general different from the distribution entailed by the original SCM.

The most common type of interventions are the so-called ``atomic interventions'', where one sets a variable to a chosen value, effectively replacing the distribution of the intervened variable with a point mass distribution.
In particular, this means that the intervened variable has no parents after the intervention.
This is the only type of intervention that we will need to consider in this work.
Formally\footref{fn:def} \cite{peters2017elements}:
\begin{definition}[Atomic intervention]
  Let $\cC = (\mathbf{X}, \mathbf{N}, S, p_{\mathbf{N}})$ be an SCM, $X_{i} \in \mathbf{X}$ and $x\in R_{X_{i}}$.
  The \emph{atomic intervention} $\dop(X_{i}=x)$ is the function $\cC(\mathbf{X}) \to \cC(\mathbf{X})$ given by $\cC \mapsto \cC^{\dop(X_{i} = x)}$, where $\cC^{\dop(X_{i} = x)}$ is the SCM that differs from $\cC$ only in that the structural assignment $f_{X_{i}}(\PA_{X_{i}}, N_{X_{i}})$ is replaced by the structural assignment $\tilde{f}_{X_{i}}(\tilde{N}_{X_{i}}) = \tilde{N}_{X_{i}}$, where $\tilde{N}_{X_{i}}$ is a random variable with range $R_{X_{i}}$ and\footnotemark $p_{\tilde{N}_{X_{i}}}(x_{i}) = \mathbf{1}_{x}(x_{i})$ for all $x_{i} \in R_{X_{i}}$.
\footnotetext{We denote by $\mathbf{1}_{x}$ the indicator function of $x$, so that $\mathbf{1}_{x}(x_{i}) =
  \begin{cases}
    1,&  x_{i} = x \\
    0,& \mathrm{otherwise}
  \end{cases}
  $. }
  Such SCM is called the \emph{post-atomic-intervention SCM}.
  One says that the variable $X_{i}$ was \emph{(atomically) intervened on}.
  The distribution $p^{\dop(X_{i} = x)} \vcentcolon= p^{\cC^{\dop(X_{i}=x)}}$ entailed by $\cC^{\dop(X_{i}= x)}$ is called the \emph{post-intervention distribution (w.r.t. the atomic intervention $\dop(X_{i} = x)$ on $\cC$)}.
\end{definition}

We can also define what we mean by ``$X$ having a total causal effect on $Y$''.
Following \cite{peters2017elements,pearl2009causality}, there is such a total causal effect if there is an atomic intervention on $X$ which modifies the initial distribution of $Y$\footref{fn:def} \cite{peters2017elements}:

\begin{definition}[Total Causal Effect]
  Let $X$, $Y$ be random variables of an SCM $\mathcal{C}$.
  $X$ has a \emph{total causal effect on} $Y$, denoted by $X \veryshortrightarrow Y$, if there is $x\in R_{X}$ such that $p^{\dop(X=x)}_{Y} \ne p_{Y}$.
\end{definition}

In this work, all variables of the form $X_{i}$, $Y_{i}$ or $Z_{i}$ are taken to be endogenous variables of some SCM $\mathcal{C}$.

\subsection{Entropy and Mutual Information}
\label{sec:entr-mutu-inform}

Since the quantities defined and studied in this article build upon the standard entropy and mutual information, it is important for the reader to be familiar with these.
In this subsection we will state the definitions of entropy, conditional entropy and mutual entropy.
In the interest of space, we will not try to motivate these definitions.
For a pedagogical introduction the reader is referred to \cite{thomas2006information,mackay2003information}.
We will also clarify what we precisely mean by causal control.

\begin{definition}[Entropy and Conditional Entropy \cite{thomas2006information}]
  \label{def:entr}
  Let $X$ be a discrete random variable with range $R_{X}$ and $p$ be a probability distribution for $X$.
  The \emph{entropy of $X$ w.r.t. the distribution $p$} is\footnotemark
  \footnotetext{In this article, $\log$ denotes the logarithm to the base $2$.}
  \begin{equation}
    \label{eq:entr}
    H_{X \sim p}(X) \vcentcolon= -\sum_{x\in R_{X}} p(x) \log p(x).
  \end{equation}
  Entropy is measured in $\mathrm{bit}$.
  If the context suggests a canonical probability distribution for $X$, one can write $H(X)$ and refers to it simply as the \emph{entropy of $X$}. \\
  The \emph{conditional entropy} $H(Y\mid X)$ of $Y$ conditioned on $X$ is the expected value w.r.t. $p_{X}$ of the entropy $H(Y \mid X=x)\vcentcolon=H_{Y\sim p_{Y\mid X=x}}(Y)$:%
  \begin{equation}
    \label{eq:cond-entr}
    H(Y\mid X) \vcentcolon= \E_{x\sim p_{X}} \left[ H(Y \mid X=x) \right].
  \end{equation}
\end{definition}

This means that the conditional entropy $H(Y \mid X)$ is the entropy of $H(Y)$ that remains on average if one conditions on $X$.

An essential concept closely associated with entropy is that of ``uncertainty.''
This qualitative concept is often present when interpreting information-theoretical quantities.
The entropy of a variable $X$ purports to measure the uncertainty regarding $X$.
In this paper, we use another qualitative concept called ``causal control'' (or simply ``control'').
The (causal) control that variable $X$ has over variable $Y$ is the level of uncertainty remaining about $Y$ after intervening on $X$.
It indicates how close we are to fully specifying $Y$ by intervening on $X$.
This understanding of the term ``control'' has been implicitly utilized in the philosophy of science literature \cite{pocheville2015comparing,bourrat2019variation}.

\begin{remark}
  Notice that $H(Y \mid X=x)$ is seen as a function of $x$ and the expected value in \Cref{eq:cond-entr} is taken over the random variable $x$ with distribution $p_{X}$.
  This disrespects the convention that random variables are represented by capital letters, but preserves the convention that the specific value conditioned upon (even if that value can be randomly realized --- \emph{i.e.} is a random variable) is represented by a lower case letter.
  Since we cannot respect both, we will follow the common practice and opt to use lower case letters for random variables in these cases.
\end{remark}

There are two common equivalent ways to define mutual information (often called information gain).

\begin{definition}[Mutual Information \cite{thomas2006information}]
  \label{def:mutual-information}
  Let $X$ and $Y$ be discrete random variables with ranges $R_{X}$ and $R_{Y}$ and distributions $p_{X}$ and $p_{Y}$, respectively.
  The \emph{mutual information} between $X$ and $Y$ is the KL divergence between the joint distribution $p_{X, Y}$ and the product distribution $p_{X}p_{Y}$, \emph{i.e.}:
  \begin{equation}
    \label{eq:mi-independence-form}
    I(X; Y) \vcentcolon= \!\!\!\! \sum_{x, y \in R_{X} \times R_{Y}} \!\!\!\! p_{X, Y}(x, y) \log \frac{p_{X, Y}(x, y)}{p_{X}(x) p_{Y}(y)}.
  \end{equation}
  Or equivalently:
  \begin{align}
    \label{eq:mi-entr-form}
    I(X; Y) &\vcentcolon= H(Y) - H(Y \mid X) \\
        &\phantom{:}= H(X) - H(X \mid Y) \nonumber.
  \end{align}
\end{definition}

The view of mutual information as entropy reduction from  \Cref{eq:mi-entr-form} will be the starting point for our definition of causal information gain.

\section{Running Example - Comparing Control Over an Outcome}
\label{sec:example}
We provide a simple example showcasing how the standard entropy and mutual information can fail to assess which variable gives us more control over a chosen outcome variable.
We will later (\Cref{sec:causal-inf-gain}) check that using causal entropy and causal information gain enable us to correctly make this assessment.

\begin{example}
  \label{ex:example-scm}

  Let us consider an ice-cream shop where the sales volume $Y$ on a given day can be categorized as low ($Y=0$), medium ($Y=1$), or high ($Y=2$).
  We would like to find a way to control $Y$.
  Assume that the sales volume is influenced by two factors: the temperature $W$, characterized as warm ($W=1$) or cold ($W=0$), and whether the ice-cream shop is being advertised, represented by the binary variable $X_{2}$.
  Additionally, we introduce a discrete variable $X_{1}$ to represent the number of individuals wearing shorts, which can be categorized as few ($X_{1}=0$), some ($X_{1}=1$), or many ($X_{1}=2$).
  Naturally, higher temperatures have a positive influence on the variable $X_{1}$.
  We do not consider any other variables.

  One can crudely model this situation using an SCM with endogenous variables $X_{1}, X_{2}, W$ and $Y$, as specified in \Cref{fig:example-scm}.
  The chosen structural assignments and noise distributions reflect the specific scenario where: the temperature $W$ is warm about half of the time; the number $X_{1}$ of people wearing shorts is highly determined by the weather conditions; and the ice-cream shop is advertised occasionally.
  $W$, $X_{2}$ and all noise variables of the SCM are binary variables, while $X_{1}, Y \in \{0, 1, 2\}$.
  Assume we cannot intervene on $W$.
  We would like to decide which of the variables $X_{1}$ or $X_{2}$ provide us with the most control over $Y$.
  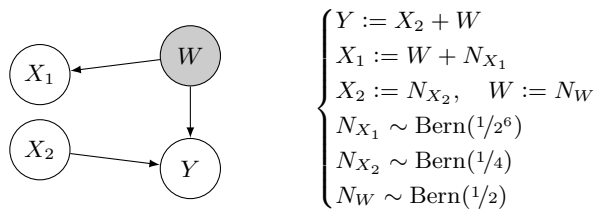
\begin{figure}[!h]
    \centering
    \begin{tikzpicture}[mynode/.style={circle,draw=black,fill=white,inner sep=0pt,minimum size=0.8cm}]
        \node[mynode] (x1) at (0,-0.25) { $X_{1}$};
        \node[mynode] (x2) at (0,-1.25) { $X_{2}$};
        \node[mynode] (y) at (2,-1.5) { $Y$};
        \node[circle,draw=black,fill=gray!40,inner sep=0pt,minimum size=0.8cm] (w) at (2,0) { $W$};
        \path [draw,->] (x2) edge[-latex] (y);
        \path [draw,->] (w) edge[-latex] (x1);
        \path [draw,->] (w) edge[-latex] (y);
        \node (assigns) at (5.7,-0.7)
          {
            $
            \begin{cases}
              Y \vcentcolon= X_{2} + W \\
              X_{1} \vcentcolon= W + N_{X_{1}} \\
              X_{2} \vcentcolon= N_{X_{2}},\quad
              W \vcentcolon= N_{W} \\
              N_{X_{1}} \sim \mathrm{Bern}(\nicefrac{1}{2^{6}}) \\
              N_{X_{2}} \sim \mathrm{Bern}(\nicefrac{1}{4}) \\
              N_{W} \sim \mathrm{Bern}(\nicefrac{1}{2}) \\
            \end{cases}
            $
          };
    \end{tikzpicture}
    \caption[Caption for SCM]{An SCM\footnotemark. It models the real-world scenario described in \Cref{ex:example-scm}, where $Y$ is the sales volume of an ice-cream shop, $W$ is the temperature, $X_{1}$ is the amount of people wearing shorts, and $X_{2}$ stands for the advertisement efforts of the ice-cream shop. The notation $N_{Z} \sim \mathrm{Bern}(q)$ signifies that the random variable $Z$ follows the Bernoulli probability distribution with parameter $q$. Grayed out variables cannot be intervened on.}
    \label{fig:example-scm}
  \addtocounter{footnote}{-1} %
  \end{figure}
  \footnotetext{The careful reader may notice that there is no noise variable $N_{Y}$ for $Y$, which seems to conflict with \Cref{def:scm}. Such apparent conflicts are resolved by seeing a deterministic assignment function such as $Y := X_{2} + W$ as having a trivial additive dependence on a noise variable $N_{Y}$ with a point mass distribution at $0$.}
  It is clear that being able to intervene on $X_{1}$ gives us no control whatsoever over $Y$.
  Any observed statistical dependence between $X_{1}$ and $Y$ comes purely from the confounder $W$.
  Consequently, interpreting a non-zero correlation or mutual information between $X_{1}$ and $Y$ as indicative of a causal connection between these variables would be a mistake, and an instance of conflation between correlation and causation.

  If we naively use the mutual information to assess whether one should intervene on $X_{1}$ or $X_{2}$ for controlling $Y$, one wrongly concludes that one should use $X_{1}$.
  Intuitively, this happens because knowing $X_{2}$ provides us with less information about $Y$ than $W$, and $X_{1}$ is very close to $W$.
  The (approximate) values can be consulted\footnotemark~in \Cref{table:mi-example}.
  \footnotetext{The details of the computations can be found in \Cref{appendix:computations_example}.}

  \begin{table}[h!]
    \caption{Information theoretical values for \Cref{fig:example-scm}.}
    \label{table:mi-example}
    \centering
    \begin{tabular}{ccccc}
      \toprule %
      $H(Y) \approx 1.41$ &$\quad$& $H(Y \mid X_{1}) \approx 0.85$ &$\quad$& $H(Y \mid X_{2}) = 1$ \\
      \midrule
      \midrule
      $I(Y ; W) \approx 0.60$ && $I(Y ; X_{1}) \approx 0.56$ && $I(Y ; X_{2}) \approx 0.41$ \\
      \bottomrule %
    \end{tabular}
  \end{table}

  Notice that $I(Y ; W) > I(Y ; X_{1})$, as it should be: $W$ has more information about $Y$ than $X_{1}$ has.
  We also see that $I(Y; X_{2}) <  I(Y ; X_{1})$.
  If mutual information were a suitable criterion for selecting the variable to intervene on, the contrary would be expected.
  In the context of our real-world scenario, intervening on the number $X_{1}$ of people wearing shorts would not be a logical approach for controlling ice cream sales.
  Instead, allocating more resources to advertising efforts (represented by $X_{2}$) would be more appropriate.

  The issue is that the mutual information $I(Y ; X_{1})$ includes the information that one has about $Y$ by \emph{observing} $X_{1}$ which flows through the confounder $W$.
  But what we want is a metric quantifying how much control we can have over $Y$ by \emph{intervening} on $X_{1}$.
  We will see that the generalization of mutual information studied in this paper (``causal information gain'') satisfies these requirements.

\end{example}

\section{Causal Entropy}
\label{sec:causal-entr}

The causal entropy of $Y$ for $X$ will be the entropy of $Y$ that is left, on average, after one atomically intervenes on $X$.
In this section we give a rigorous definition of causal entropy and study its connection to causal effect.

We define causal entropy in a manner analogous to conditional entropy (see \Cref{def:entr}).
It will be the average uncertainty one has about $Y$ if one sets $X$ to $x$ with probability $p_{X'}(x)$, where $X'$ is a new auxiliary variable with the same range as $X$ but independent of all other variables, including $X$.
In contrast with the non-causal case, here one needs to make a choice of distribution over $X'$ corresponding to the distribution over the atomic interventions that one is intending to perform.

\begin{definition}[Causal entropy, $H_{c}$]
  Let $Y$, $X$ and $X'$ be random variables such that $X$ and $X'$ have the same range and $X'$ is independent of all variables in $\cC$.
  We say that $X'$ is an \emph{intervention protocol} for $X$.

  The \emph{causal entropy} $H_{c}(Y\mid \dop(X \sim  X'))$ of $Y$ given the intervention protocol $X'$ for $X$ is the expected value w.r.t. $p_{X'}$ of the entropy $H(Y \mid \dop(X = x)) \vcentcolon= H_{Y \sim p_{Y}^{\dop(X=x)}}(Y)$ of the interventional distribution $p_{Y}^{\dop(X=x)}$.
  That is:
  \begin{equation}
    \label{eq:caus-cond-entr}
    H_{c}(Y\mid \dop(X \sim  X')) \vcentcolon= \E_{x\sim p_{X'}} \left[ H(Y \mid \dop(X=x)) \right]
  \end{equation}
\end{definition}

We will now see that, unsurprisingly, if there is no total effect of $X$ on $Y$, then the causal entropy is just the initial entropy $H(Y)$.
Perhaps more unexpectedly, the converse is not true:
it is possible to have $H_{c}(Y \mid X \sim X') = H(Y)$ while $X \veryshortrightarrow Y$.
One way this can happen is due to the non-injectivity of entropy when seen as a mapping from the set of distributions over $Y$, \emph{i.e.} it may happen that $p^{\dop(X=x)}_{Y}\ne p_{Y}$ but $H_{Y\sim p^{\dop(X=x)}_{Y}}(Y) = H_{Y \sim p_{Y}}(Y)$.

\begin{proposition}
  \label{prop:no-total-effect-implies-hc-eq-h}
  If there is no total effect of $X$ on $Y$, then $H_{c}(Y \mid \dop(X \sim X'))=H(Y)$ for any intervention protocol $X'$ for $X$.
  The converse does not hold.
\end{proposition}

\begin{proof}
  The proof can be found in \Cref{appendix:proof_no_total_effect_implies_hc_eq_h}. \qed
\end{proof}

If there is a total causal effect of $X$ on $Y$, there cannot be a total causal effect of $Y$ on $X$ (if $X$ is a cause of $Y$, $Y$ cannot be a cause of $X$) \cite{peters2017elements}.
This immediately yields the following corollary.

\begin{corollary}
  \label{cor:hc-invariant-if-other-is-not}
  If $H_{c}(Y \mid \dop(X \sim X')) \ne H(Y)$ for some intervention protocol $X'$ for $X$, then $H_{c}(X \mid \dop(Y \sim Y')) = H(X)$ for any intervention protocol $Y'$ for $Y$.
\end{corollary}
\begin{proof}
  \label{proof:hc-invariant-if-other-is-not}
  Suppose that $H_{c}(Y \mid X \sim X') \ne H(Y)$.
  By the contrapositive of \Cref{prop:no-total-effect-implies-hc-eq-h}, this means that there is a total effect of $X$ on $Y$.
  Hence there is no total effect of $Y$ on $X$, which again by \Cref{prop:no-total-effect-implies-hc-eq-h} yields the desired result. \qed
\end{proof}

\section{Causal Information Gain}
\label{sec:causal-inf-gain}
Causal information gain extends mutual information to the causal context.
The causal information gain of $Y$ for $X$ will be the average decrease in the entropy of $Y$ after one atomically intervenes on $X$.
We start this section by giving a rigorous definition of causal information gain, and proceed to study its connection with causal effect.
We end this section by revisiting \Cref{ex:example-scm} armed with this new information theoretical quantity.
We will confirm in this example that causal information is the correct tool for assessing which variable has the most causal control over the outcome, as opposed to standard mutual information.

Recall the entropy-based definition of mutual information in \Cref{eq:mi-entr-form}.
The mutual information between two variables $X$ and $Y$ is the average reduction in uncertainty about $Y$ if one observes the value of $X$ (and vice-versa, by symmetry of the mutual information).
This view of mutual information allows for a straightforward analogous definition in the causal case, so that one can take causal information gain $I_{c}(Y \mid \dop(X\sim X'))$ to signify the average reduction in uncertainty about $Y$ if one sets $X$ to $x$ with probability $p_{X'}(x)$.

\begin{definition}[Causal Information Gain, $I_{c}$]
  \label{def:Ic}
  Let $Y$, $X$ and $X'$ be random variables such that $X'$ is an intervention protocol for $X$.
  The \emph{causal information gain} $I_{c}(Y\mid \dop(X \sim  X'))$ of $Y$ for $X$ given the intervention protocol $X'$ is the difference between the entropy of $Y$ w.r.t. its prior and the causal entropy of $Y$ given the intervention protocol $X'$.
  That is:
  \begin{equation}
    \label{eq:caus-cond-entr}
    I_{c}(Y \mid \dop(X \sim X')) \vcentcolon= H(Y) - H_{c}(Y\mid \dop(X \sim  X')).
  \end{equation}
\end{definition}

A few properties of causal information gain can be immediately gleaned from its definition.
First, in contrast with mutual information, causal information gain is \emph{not} symmetric.
Also, similarly to causal entropy, one needs to specify an intervention protocol with a distribution to be followed by interventions on $X$.

We can make use of the relation between causal entropy and causal effect to straightforwardly deduce the relation between causal information gain and causal effect.

\begin{proposition}
  \label{cor:non-zero-Ic-means-causal-effect}
    If $I_{c}(Y \mid \dop(X \sim X'))\ne 0$ for some protocol $X'$ for $X$, then $X \veryshortrightarrow Y$.
    The converse does not hold.
\end{proposition}
\begin{proof}
  The implication in this proposition follows directly from \Cref{def:Ic} and the contrapositive of the implication in \Cref{prop:no-total-effect-implies-hc-eq-h}.
  The converse does not hold simply because it is equivalent to the converse of the contrapositive of the implication in \Cref{prop:no-total-effect-implies-hc-eq-h}, which also does not hold. \qed
\end{proof}

\begin{corollary}
  \label{cor:one-Ic-is-zero}
  Let $X'$ and $Y'$ be intervention protocols for $X$ and $Y$, respectively.
  At least one of $I_{c}(Y \mid \dop(X \sim X'))$ or $I_{c}(X \mid \dop(Y \sim Y'))$ is zero.
\end{corollary}
\begin{proof}
  Suppose both $I_{c}(Y \mid \dop(X \sim X'))$ and $I_{c}(X \mid \dop(Y \sim Y'))$ are non-zero.
  Then by \Cref{cor:non-zero-Ic-means-causal-effect} we have both $X \veryshortrightarrow Y$ and $Y \veryshortrightarrow X$, which is not possible in the context of an SCM. \qed
\end{proof}

It is worth noting that the last part of \Cref{cor:non-zero-Ic-means-causal-effect} contradicts \cite{pocheville2015comparing}.
In that work, it is stated without proof that ``causation is equivalent to non-zero specificity'', wherein the term ``specificity'' coincides with what we refer to as causal information gain given a uniformly distributed intervention protocol.

\subsection{Comparison of Causal Information Gain and Mutual Information in Running Example}
\label{sec:app-macro-selec}

Consider again \Cref{ex:example-scm}.
Compare the causal entropy and causal information gain values\footnotemark in \Cref{table:hc-example} with the conditional entropy and mutual information values from \Cref{table:mi-example}.
\footnotetext{In this particular case it does not matter what intervention protocol $X'$ we choose, since $H_{c}(Y \mid \dop(X_{1} = x_{1})) = H(Y) \approx 1.41$ for all $x_{1}$ and $H_{c}(Y \mid \dop(X_{2} = x_{2})) = 1$ for all $x_{2}$.}

  \begin{table}[h!]
      \caption{Causal information theoretical values for \Cref{fig:example-scm}.}
      \label{table:hc-example}
      \centering
      \begin{tabular}{ccc}
        \toprule %
        $H_{c}(Y \mid \dop(X_{1} \sim X'_{1})) \approx 1.41 $ &$\quad$& $H_{c}(Y \mid \dop(X_{2} \sim X'_{2})) = 1$ \\
        \midrule %
        \midrule
        $I_{c}(Y \mid \dop(X_{1} \sim X'_{1})) = 0$ && $I_{c}(Y \mid \dop(X_{2} \sim X'_{2})) \approx 0.41$ \\
        \bottomrule %
      \end{tabular}
  \end{table}

  We see that using causal information gain allows us to correctly conclude that using $X_{1}$ to control $Y$ would be fruitless: intervening on $X_{1}$ does not change the entropy of $Y$.
  This is reflected by the fact that the causal information gain of $Y$ for $X_{1}$ is zero.
  Since $X_{1}$ has no causal effect on $Y$, this result was to be expected by the contrapositive of \Cref{cor:non-zero-Ic-means-causal-effect}.
  On the other hand, $X_{2}$ does provide us with some control over $Y$: intervening on $X_{2}$ decreases the entropy of $Y$ by $0.4 \text{ bit}$ on average.
  In the real-world scenario described in \Cref{ex:example-scm}, utilizing causal information gain to determine which variable to intervene on for controlling the sales volume $Y$ would lead us to make the correct decision of intensifying advertising efforts ($X_{2}$).
  Furthermore, it would enable us to conclude that manipulating the number of people wearing shorts ($X_{1}$) provides no control whatsoever over $Y$.
  Thus, causal information gain could be used in this case to assess whether statistical dependence between $Y$ and another variable in this causal system can be interpreted to have causal significance.

\section{Related Work}
\label{sec:related-work}

Previous work has aimed to provide causal explanations of machine learning models through ``counterfactual explanations'' \cite{wachter2017counterfactual,mothilal2020explaining}.
These explanations reveal what the model would have predicted under different feature values.
However, they do not offer insights into the causal significance of a feature in influencing the outcome variable.
Instead, they merely inform us about the behavior of the model itself.
In other words, counterfactual explanations inform us about the changes required for the model to produce a different prediction, but not the changes necessary for the outcome to differ in reality.
While counterfactual explanations can be useful, for instance, in advising loan applicants on improving their chances of approval \cite{mothilal2020explaining}, they fall short in providing causal interpretations for tasks such as scientific exploration \cite{zednik2022scientific}, where it is crucial to understand the actual causal relationships between features and the chosen outcome.
As discussed in \Cref{sec:introduction}, the quantities investigated in this paper can precisely address this need.

Information theoretical quantities aimed at capturing aspects of causality have been previously proposed.
An important example is the work in \cite{janzing2013quantifying}.
In that paper, the authors suggest a list of postulates that a measure of causal strength should satisfy, and subsequently demonstrate that commonly used measures fall short of meeting them.
They then propose their own measure (called ``causal influence''), which does satisfy the postulates.
Causal influence is the KL divergence of the original joint distribution and the joint distribution resulting from removing the arrows whose strength we would like to measure, and feeding noise to the orphaned nodes.
Thus although it utilizes information theory, it does not purport to generalize entropy or mutual information to the causal context.
One information-theoretical measure mentioned in \cite{janzing2013quantifying} is closer to ours.
It is called ``information flow'' \cite{ay2008information}.
Similarly to causal information gain, this quantity is a causal generalization of mutual information.
Their goal was to come up with a generalization of mutual information that would be a measure of ``causal independence'' in much the same way as standard mutual information is a measure of statistical independence.
They take the route of starting from the definition of mutual information as the KL divergence between the joint distribution and the product of the marginal distributions (\Cref{eq:mi-independence-form}), and proceed to ``make it causal'' by effectively replacing conditioning with intervening everywhere.
In contrast, we treat entropy as the main quantity of interest, and start from the definition of mutual entropy as the change in entropy due to conditioning (\Cref{eq:mi-entr-form}), and proceed to define its causal counterpart as the change in entropy due to intervening.
This then results in a quantity that is the appropriate tool for evaluating the control that a variable has over another.

The basic idea of extending the concept of mutual information to the causal context as the average reduction of entropy after intervening was introduced in the philosophy of science literature, as part of an attempt to capture a property of causal relations which they refer to as ``specificity'' \cite{griffiths2015specificity}.
This property can be thought of as a measure of the degree to which interventions on the cause variable result in a deterministic one-to-one mapping \cite{woodward2010causation}.
This means that maximal specificity of a causal relationship is attained when: (a) performing an atomic intervention on the cause variable results in complete certainty about the effect variable's value; and (b) no two distinct atomic interventions on the cause variable result in the same value for the effect variable \cite{griffiths2015specificity}.
Notice that (a) means precisely that the cause variable provides maximal causal control over the effect variable.
The causal extension of mutual information proposed in \cite{griffiths2015specificity} was named ``causal mutual information''.
They call ``causal entropy'' the average entropy of the effect variable after performing an atomic intervention on the cause variable.
Their ``causal mutual information'' is then the difference between the initial entropy of the effect variable and the causal entropy.
Although they do not say so explicitly, their definition of causal entropy assumes that one only cares about the entropy that results from interventions that are equally likely: the average of post-intervention entropies is taken w.r.t. a uniform distribution --- hence their ``causal entropy'' is the same as the causal entropy defined in this paper, but restricted to uniform intervention protocols.
This was also noted in \cite{pocheville2015comparing}, where the authors propose that other choices of distribution over the interventions would result in quantities capturing causal aspects that are distinct from the standard specificity.
In this paper we both generalized and formalized the information theoretical notions introduced in \cite{griffiths2015specificity}.
We provided rigorous definitions of causal entropy and causal information gain which allow for the use of non-uniform distributions over the interventions.
Our causal entropy can thus be seen as a generalized version of their causal entropy, while our causal information gain can be seen as a generalization of their causal mutual information\footnotemark.
\footnotetext{The term causal mutual information may be misleading given the directional nature of the relationship between cause and effect. We thus prefer the term causal information gain, drawing inspiration from the alternate name ``information gain'', which is frequently employed in discussions about decision trees when referring to mutual information.}
Armed with concrete, mathematical definitions, we are able to study key mathematical aspects of these quantities.

\section{Discussion and Conclusion}
\label{sec:disc-concl}
The motivation behind extending traditional entropy and mutual information to interventional settings in the context of interpretable machine learning (IML) arises from the necessity to determine whether the high importance assigned to specific features by machine learning models and IML methods can be causally interpreted or is purely of a statistical nature.

Information theoretical quantities are commonly used to assess statistical feature importance.
We extended these quantities to handle interventions, allowing them to capture the control one has over a variable by manipulating another.
The proposed measures, namely causal entropy and causal information gain, hold promise for the development of new algorithms in domains where knowledge of causal relationships is available or obtainable.
It is worth noting that the utility of these measures extends well beyond the field of IML, as both information-theoretical quantities and the need for causal control are pervasive in machine learning.

Moving forward, a crucial theoretical endeavor involves establishing a fundamental set of properties for the proposed causal information-theoretical measures.
This can include investigating a data processing inequality and a chain rule for causal information gain, drawing inspiration from analogous properties associated with mutual information.
Other important research directions involve the extension of these definitions to continuous variables, as well as investigating the implications of employing different intervention protocols.
Furthermore, the design and study of appropriate estimators for these measures constitute important avenues for future research, as well as their practical implementation.
Ideally, these estimators should be efficient to compute even when dealing with high-dimensional data and complex, real-world datasets.
Additionally, they ought to be applicable to observational data.
In cases where the structural causal model is known, this could be accomplished by utilizing a framework such as \textit{do}-calculus \cite{pearl2009causality} when devising the estimators.
This could allow for their application in extracting causal insights from observational data.

\appendix

\section{Computations for the running example}
\label{appendix:computations_example}

We have
\begin{align*}
  H(Y) &= p_{Y}(0) \log(\frac{1}{p_{Y}(0)}) + p_{Y}(1) \log(\frac{1}{p_{Y}(1)}) + p_{Y}(2) \log(\frac{1}{p_{Y}(2)}) \\
       &= \frac{3}{8} \log(\frac{8}{3}) + \frac{1}{2} \log(2) + \frac{1}{8} \log(8) = 2 - \frac{3}{8} \log(3) \approx 1.41 \,\mathrm{(bit)},
\end{align*}

and

\begin{align*}
  H(Y \mid W) &= H(Y \mid W=0) = \frac{3}{4} \log(\frac{4}{3}) + \frac{1}{4} \log(4) \approx 0.81 \,\mathrm{(bit)},
\end{align*}
where we used that $H(Y \mid W=0) = H(Y \mid W=1)$, so that taking the average is unnecessary.

Notice that $X_{1} = 0$ implies $W = 0$, in which case $Y = X_{2}$.
Hence $H(Y \mid X_{1} = 0) = H(Y \mid W=0) \approx 0.81 \,\mathrm{(bit)}$.
By a similar argument, $H(Y \mid X_{1}=2) = H(Y \mid W=1) \approx 0.81 \,\mathrm{(bit)}$.
Now, denote $q = \frac{1}{64}$.
It is easy to check that $p_{Y \mid X_{1}=1}(0) = \frac{3q}{4}$, $p_{Y \mid X_{1}=1}(1) = \frac{3}{4} - \frac{q}{2}$ and $p_{Y \mid X_{1}=1}(2) = \frac{1}{4} (1-q)$.
Then
\begin{equation*}
  H(Y\mid X_{1} = 1) = - \frac{3q}{4} \log(\frac{3q}{4}) - (\frac{3}{4} - \frac{q}{2}) \log(\frac{3}{4} - \frac{q}{2}) - \frac{1}{4} (1-q) \log(\frac{1}{4} (1-q)) \approx 0.89 \,\mathrm{(bit)}.
\end{equation*}

We can then compute:
\begin{align*}
  H(Y \mid X_{1}) &= p_{X_{1}}(0) \overbrace{H(Y \mid X_{1} = 0)}^{0.81} + p_{X_{1}}(1) \overbrace{H(Y \mid X_{1} = 1)}^{0.89} + p_{X_{1}}(1) \overbrace{H(Y \mid X_{1} = 2)}^{0.81} \\
                  &= \frac{1}{2} \times (1 - q) \times 0.81 + \frac{1}{2} \times 0.89 + \frac{q}{2} \times 0.81 \approx 0.85 \,\mathrm{(bit)}.
\end{align*}

We also have:
\begin{equation*}
  H(Y \mid X_{2}) = p_{X_{2}}(0) \overbrace{H(Y \mid X_{2} = 0)}^{1} + p_{X_{2}}(1) \overbrace{H(Y \mid X_{2} = 1)}^{1} = 1 \,\mathrm{(bit)},
\end{equation*}

It immediately follows that $I(Y ; W) \approx 0.60$, $I(Y ; X_{1}) \approx 0.56\,\mathrm{(bit)}$ and $I(Y ; X_{2}) \approx 0.41\,\mathrm{(bit)}$.

Moving on to the causal information theoretical quantities, we have
$H(Y \mid \dop(X_{1}=x_{1})) = H(Y) \approx 1.41 \,\mathrm{(bit)}$ for every $x_{1} \in R_{X_{1}}$ and $H(Y \mid \dop(X_{2}=x_{2})) = H(W) = 1 \,\mathrm{(bit)}$ for every $x_{2} \in R_{X_{2}}$.
Hence $H_{c}(Y \mid \dop(X_{1} \sim X_{1}')) \approx 1.41 \,\mathrm{(bit)}$ and $H_{c}(Y \mid \dop(X_{2} \sim X_{2}')) = 1\,\mathrm{(bit)}$ for any intervention protocols $X'_{1}, X'_{2}$.
It follows that $I_{c}(Y \mid \dop(X_{1} \sim X_{1}')) = 0\,\mathrm{(bit)}$ and $I(Y \mid \dop(X_{2} \sim X_{2}')) \approx 0.41\,\mathrm{(bit)}$.

\section{Proof of \Cref{prop:no-total-effect-implies-hc-eq-h}}
\label{appendix:proof_no_total_effect_implies_hc_eq_h}

\begin{proof}
  \label{proof:no-total-effect-implies-hc-eq-h}
  Suppose $X$ has no causal effect on $Y$.
  Then $\forall x\in R_{X},\  p^{\dop(X=x)}_{Y} \!\!= p_{Y}$.
  The expression for the causal entropy then reduces to $\E_{x\sim X'} H(Y) = H(Y)$.
  This shows the implication in the proposition.

  We will check that the converse does not hold by giving an example where $X$ has a causal effect on $Y$ but $H_{c}(Y \mid X \sim X') = H(Y)$.
  Consider the SCM with three binary endogenous variables $X, Y$ and $M$ specified by:
  \begin{equation}
    \begin{cases}
      f_{M}(N_{M}) = N_{M} \\
      f_{X}(M, N_{X}) =
      \begin{cases}
        (N_{X} + 1) \mod 2, M=1 \\
        N_{X}, M=0
      \end{cases} \\
      f_{Y}(M, N_{X}) =
      \begin{cases}
        X, M=1 \\
        (X + 1) \mod 2, M=0
      \end{cases} \\
      N_{X}, N_{M} \sim \mathrm{Bern}(q),
  \text{ for some }q \in (0,1).
    \end{cases}
  \end{equation}
  Then $p_{Y}^{\dop(X=0)} \sim \mathrm{Bern}(q)$ and $p_{Y}^{\dop(X=1)} \sim \mathrm{Bern}(q)$.
  Also,
  \begin{equation}
    p_{Y} = p_{X \mid M=1}(1) p_{M}(1) + p_{X \mid M=0}(0) p_{M}(0) = 1-q \quad \Rightarrow \quad Y \sim \mathrm{Bern}(1-q)
  \end{equation}
  Hence $p_{Y} \ne p_{Y}^{\dop(X=1)}$, meaning that $X \veryshortrightarrow Y$.
  And since both post-intervention distributions have the same entropy $H_{Y\sim \mathrm{Bern}(q)}(Y) = H_{Y\sim \mathrm{Bern}(1-q)}(Y)$, then the causal entropy will also be $H_{c}(Y \mid X \sim X') = H_{Y\sim \mathrm{Bern}(1-q)}(Y) = H(Y)$ (for any chosen of $X'$). \qed
\end{proof}

\bibliographystyle{splncs04}
\bibliography{ecai23library.bib}

\begin{thebibliography}{10}
\providecommand{\url}[1]{\texttt{#1}}
\providecommand{\urlprefix}{URL }
\providecommand{\doi}[1]{https://doi.org/#1}

\bibitem{ay2008information}
Ay, N., Polani, D.: Information flows in causal networks. Advances in complex
  systems  \textbf{11}(01),  17--41 (2008)

\bibitem{beraha2019feature}
Beraha, M., Metelli, A.M., Papini, M., Tirinzoni, A., Restelli, M.: Feature
  selection via mutual information: New theoretical insights. CoRR
  \textbf{abs/1907.07384} (2019), \url{http://arxiv.org/abs/1907.07384}

\bibitem{bourrat2019variation}
Bourrat, P.: Variation of information as a measure of one-to-one causal
  specificity. European Journal for Philosophy of Science  \textbf{9}(1),
  1--18 (2019)

\bibitem{confalonieri2021historical}
Confalonieri, R., Coba, L., Wagner, B., Besold, T.R.: A historical perspective
  of explainable artificial intelligence. Wiley Interdisciplinary Reviews: Data
  Mining and Knowledge Discovery  \textbf{11}(1),  e1391 (2021)

\bibitem{thomas2006information}
Cover, T.M., Thomas, J.A.: Elements of information theory. Wiley-Interscience
  (2006)

\bibitem{friedman2001pdp}
Friedman, J.H.: Greedy function approximation: a gradient boosting machine.
  Annals of statistics pp. 1189--1232 (2001)

\bibitem{goldstein2015peeking}
Goldstein, A., Kapelner, A., Bleich, J., Pitkin, E.: Peeking inside the black
  box: Visualizing statistical learning with plots of individual conditional
  expectation. journal of Computational and Graphical Statistics
  \textbf{24}(1),  44--65 (2015)

\bibitem{griffiths2015specificity}
Griffiths, P.E., Pocheville, A., Calcott, B., Stotz, K., Kim, H., Knight, R.:
  Measuring causal specificity. Philosophy of science  \textbf{82}(4),
  529--555 (2015)

\bibitem{janzing2013quantifying}
Janzing, D., Balduzzi, D., Grosse-Wentrup, M., Sch{\"o}lkopf, B.: Quantifying
  causal influences. The Annals of Statistics  \textbf{41}(5),  2324--2358
  (2013)

\bibitem{koller2009probabilistic}
Koller, D., Friedman, N.: Probabilistic graphical models: principles and
  techniques. MIT press (2009)

\bibitem{mackay2003information}
MacKay, D.J., Mac~Kay, D.J.: Information theory, inference and learning
  algorithms. Cambridge university press (2003)

\bibitem{miller2019explanation}
Miller, T.: Explanation in artificial intelligence: Insights from the social
  sciences. Artificial intelligence  \textbf{267},  1--38 (2019)

\bibitem{mothilal2020explaining}
Mothilal, R.K., Sharma, A., Tan, C.: Explaining machine learning classifiers
  through diverse counterfactual explanations. In: Proceedings of the 2020
  conference on fairness, accountability, and transparency. pp. 607--617 (2020)

\bibitem{pearl2009causality}
Pearl, J.: Causality. Cambridge university press (2009)

\bibitem{pearl2016primer}
Pearl, J., Glymour, M., Jewell, N.P.: Causal inference in statistics: A primer.
  John Wiley \& Sons (2016)

\bibitem{peters2017elements}
Peters, J., Janzing, D., Sch{\"o}lkopf, B.: Elements of causal inference:
  foundations and learning algorithms. The MIT Press (2017)

\bibitem{pocheville2015comparing}
Pocheville, A., Griffiths, P., Stotz, K.: Comparing causes – an
  information-theoretic approach to specificity, proportionality and stability.
  15th Congress of Logic, Methodology, and Philosophy of Science  (08 2015)

\bibitem{ribeiro2016lime}
Ribeiro, M.T., Singh, S., Guestrin, C.: " why should i trust you?" explaining
  the predictions of any classifier. In: Proceedings of the 22nd ACM SIGKDD
  international conference on knowledge discovery and data mining. pp.
  1135--1144 (2016)

\bibitem{scholkopf2012causal}
Sch{\"o}lkopf, B., Janzing, D., Peters, J., Sgouritsa, E., Zhang, K., Mooij,
  J.: On causal and anticausal learning. arXiv preprint arXiv:1206.6471  (2012)

\bibitem{vergara2014review}
Vergara, J.R., Est{\'e}vez, P.A.: A review of feature selection methods based
  on mutual information. Neural computing and applications  \textbf{24},
  175--186 (2014)

\bibitem{wachter2017counterfactual}
Wachter, S., Mittelstadt, B., Russell, C.: Counterfactual explanations without
  opening the black box: Automated decisions and the gdpr. Harv. JL \& Tech.
  \textbf{31}, ~841 (2017)

\bibitem{woodward2010causation}
Woodward, J.: Causation in biology: stability, specificity, and the choice of
  levels of explanation. Biology \& Philosophy  \textbf{25}(3),  287--318
  (2010)

\bibitem{zednik2022scientific}
Zednik, C., Boelsen, H.: Scientific exploration and explainable artificial
  intelligence. Minds and Machines  \textbf{32}(1),  219--239 (2022)

\bibitem{zhao2019causal}
Zhao, Q., Hastie, T.: Causal interpretations of black-box models. Journal of
  Business \& Economic Statistics  (2019)

\bibitem{zhou2022feature}
Zhou, H., Wang, X., Zhu, R.: Feature selection based on mutual information with
  correlation coefficient. Applied Intelligence pp. 1--18 (2022)

\end{thebibliography}
\end{document}